\documentclass[conference]{IEEEtran}
\usepackage{graphicx,cite,calc}
\usepackage[usenames,dvipsnames]{color}
\usepackage[center]{subfigure}
\usepackage{amsthm,amsfonts}
\usepackage{hyperref}
\usepackage{subeqn}
\usepackage{ctable}

\newtheorem{theorem}{Theorem}

\newcommand{\beq}{\begin{equation}}
\newcommand{\eeq}{\end{equation}}
\newcommand{\beqa}{\begin{eqnarray}}
\newcommand{\eeqa}{\end{eqnarray}}
\newcommand{\bit}{\begin{itemize}}
\newcommand{\eit}{\end{itemize}}
\newcommand{\ben}{\begin{enumerate}}
\newcommand{\een}{\end{enumerate}}
\newcommand{\mc}{\mathcal}
\newcommand{\mb}{\mathbb}
\newcommand{\bed}{\begin{displaymath}}
\newcommand{\eed}{\end{displaymath}}
\newtheorem{thm}{Theorem}

 \newtheorem{rem}[thm]{Remark}
\usepackage{algorithm,algorithmic}

\begin{document}

\title{Multi-Level Error-Resilient Neural Networks with Learning}

\author{\IEEEauthorblockN{Amir Hesam Salavati and Amin Karbasi}
\IEEEauthorblockA{School of Computer and Communication Sciences, Ecole Polytechnique Federale de Lausanne (EPFL)\\
Email: hesam.salavati@epfl.ch, amin.karbasi@epfl.ch,}}

\maketitle

\begin{abstract}
The problem of neural network association is to retrieve 
a previously memorized pattern from its noisy version using a
network of neurons. An ideal neural network should include three components simultaneously: a learning algorithm, a large pattern retrieval capacity and resilience against noise. Prior works in this area usually improve one or two aspects at the cost of the third. 

Our work takes a step forward in closing this gap. More specifically, we show that by forcing natural constraints on the set of learning patterns, we can drastically improve the retrieval capacity of our neural network. Moreover, we devise a learning algorithm whose role is to learn those patterns satisfying the above mentioned constraints. Finally we show that our neural network can cope with a fair amount of noise.

\end{abstract}

\section{Introduction}
Neural networks are famous for their ability to \emph{learn} and \emph{reliably} perform a required task. An important example is the case of (associative) memory where we are asked to memorize (learn) a set of given patterns. Later, corrupted versions of the memorized patterns will be shown to us and we have to return the correct memorized patterns. In essence, this problem is very similar to the one faced in communication systems where the goal is to reliably transmit and efficiently decode a set of patterns (so called codewords) over a noisy channel.

As one would naturally expect, reliability is certainly a very important issue both the in neural associative memories and in communication systems. Indeed, the last three decades witnessed many reliable artificial associative neural networks. See for instance \cite{hopfield}, \cite{venkatesh}, \cite{Jankowski}, \cite{Muezzinoglu1}, \cite{SKGS}, \cite{gripon_sparse}.

However, despite common techniques and methods deployed in both fields (e.g., graphical models, iterative algorithms, etc), there has been a quantitative difference in terms of another important criterion: the efficiency.ver the past decade, by using probabilistic graphical models in communication systems it has become clear that the number of patterns that can be reliably transmitted and efficiently decoded over a noisy channel is exponential in $n$, length of the codewords, \cite{urbanke}. However, using current neural networks of size $n$ to memorize a set of \emph{randomly} chosen patterns, the maximum number of patterns that can be reliably memorized scales linearly in $n$ \cite{mceliece}, \cite{venkatesh}.

% where patterns can be thought of codewords of a given code book. In short, the problem of interest is to "memorize" a set of vectors of length $n$ using a neural network, where the term memorize refers to ???the best pattern
%retrieval capacities that can be obtained in a neural network
%scale only linearly with the length of the patterns.

%Although there are lots of approaches in neuroscience community to accomplish the task of memorizing, they all lack in efficiency in the sense that the number of patterns that can be memorized and reliably recalled using these approaches is extremely limited. More specifically, for a network size of $n$ neurons, the number of patterns that can be memorized is $O(n)$ at best. 

There are multiple reasons for the inefficiency of the storage capacity of neural networks. First, neurons can only perform simple operations. As a result, most of the techniques used in communication systems (more specifically in coding theory) for achieving exponential storage capacity are prohibitive in neural networks. Second, a large body of past work (e.g., \cite{hopfield}, \cite{venkatesh}, \cite{Jankowski}, \cite{Muezzinoglu1}) followed a common assumption that a neural network should be able to memorize \emph{any} subset of patterns drawn randomly from the set of all possible vectors of length $n$. Although this assumption gives the network a sense of generality, it reduces its storage capacity to a great extent. 

An interesting question which arises in this context is whether one can increase the storage capacity of neural networks beyond the current linear scaling and achieve results similar to coding theory. To this end, Kumar et al. \cite{KSS} suggested a new formulation of the problem where only a suitable set of patterns was considered for storing. This way they could show that the performance of neural networks in terms of storage capacity increases significantly. Following the same philosophy, we will focus on memorizing a random subset of patterns of length $n$ such that the dimension of the training set is $k < n$. In other words, we are interested in memorizing a set of patterns that have a certain degree of \emph{structure} and \emph{redundancy}. We exploit this structure both to increase the number of patterns that can be memorized (from linear to exponential) and to increase the number of errors that can be corrected when the network is faced with corrupted inputs. 

The success of \cite{KSS} is mainly due to forming a bipartite network/graph (as opposed to a complete graph) whose role is to enforce the suitable constraints on the patterns, very similar to the role played by Tanner graphs in coding. More specifically, one layer is used to feed the patterns to the network (so called variable nodes in coding) and the other takes into account the inherent structure of the input patterns (so called check nodes in coding). A natural way to enforce structures on inputs is to assume that the connectivity matrix of the bipartite graph is orthogonal to all of the input patterns. However, the authors in \cite{KSS} heavily rely on the fact that the bipartite graph is fully known and given, and satisfies some sparsity and expansion properties. The expansion assumption is made to ensure that the resulting set of patterns are resilient against fair amount of noise. Unfortunately, no algorithm for finding such a bipartite graph was proposed.

Our main contribution in this paper is to relax the above assumptions while achieving better error correction performance. More specifically, we first propose an iterative algorithm that can find a sparse bipartite graph that satisfies the desired set of constraints. We also provide an upper bound on the block error rate of the method that deploys this learning strategy. We then proceed to devise a multi-layer network whose performance in terms of error tolerance improves significantly upon \cite{KSS} and no longer needs to be an expander. 

The remainder of this paper is organized as follows. In Section~\ref{sec:formulation} we formally state the problem that is the focus of this work, namely neural association for a network of non-binary neurons. We then provide an overview of the related work in this area in Section~\ref{sec:related}. We present our pattern learning algorithm in Section~\ref{sec:learning} and the multi-level network design in Section~\ref{sec:multi}. The simulations supporting our analytical results are shown in Section~\ref{sec:simul}. Finally future works are explained in Section~\ref{section:conclusions}.

%first computes the
%weighted sum
%\[ h = \sum_{i=1}^{n} w_i s_i, \]
%where $w_i$ denotes the weight of the input link from $s_i$, and then
%updates its state as $x = f(h),$
%where $f: \mb{R} \rightarrow \mc{S}$ is a possibly non-linear function
%from the field of real numbers $\mb{R}$ to $\mc{S}$.
\section{Problem Formulation}\label{sec:formulation}
In contrast to the mainstream work in neural associative memories, we focus on non-binary neurons, i.e., neurons that can assume a finite set of integer values $\mc{S} = \{
0,1,\dots,S-1\}$ for their states (where $S>2$). A natural way to interpret the multi-level states is to think of the short-term
(normalized) firing rate of a neuron as its output. 
Neurons can only perform simple operations. In particular, we restrict the
operations at each neuron to a \textit{linear summation} over the inputs, and
a possibly \textit{non-linear thresholding} operation. In particular, a neuron $x$ updates its state based on the states of
its neighbors $\{s_i\}_{i=1}^{n}$ as follows:
\begin{enumerate}
\item It computes the weighted sum
$ h = \sum_{i=1}^{n} w_i s_i,$ 
where $w_i$ denotes the weight of the input link from $s_i$.
\item It updates its state as $x = f(h),$
where $f: \mb{R} \rightarrow \mc{S}$ is a possibly non-linear function
from the field of real numbers $\mb{R}$ to $\mc{S}$.
\end{enumerate}
Neural associative memory aims to memorize
$C$ patterns of length $n$ by determining the weighted connectivity matrix of the neural network (\emph{learning phase}) such that the given patterns are stable states of the network. Furthermore, the network should be able to tolerate a fair amount of noise so that it can return the correct memorized pattern in response to a corrupted query (\emph{recall phase}). Among the networks with these two abilities, the one with largest $C$ is the most desirable. 

We first focus on learning the connectivity matrix of a neural graph which memorizes a set of patterns having some inherent redundancy. More specifically, we assume to have $C$ vectors of length $n$ with non-negative integer entries, where these patterns form a subspace of dimension $k < n$. We would like to memorize these patterns by finding a set of non-zero vectors $w_1,\dots,w_m \in \mb{R}^n$ that are orthogonal to the set of given patterns. Furthermore, we are interested in rather sparse vectors. Putting the training patterns in a matrix $\mathcal{X}_{C \times n}$ and focusing on one such vector $w$, we can formulate the problem as: 
\begin{subequations}\label{main_problem}
\beq\label{main_problem_obj}
\min \Vert \mathcal{X}\cdot w \Vert_2
\eeq
subject to
\beq\label{main_problem_const1}
\Vert w \Vert_0 \leq q \quad \hbox{and} \quad \Vert w \Vert_2^2 \geq \epsilon
\eeq
%and
%\beq\label{main_problem_const2}
%\Vert w \Vert_2^2 \geq \epsilon
%\eeq
\end{subequations}
where $q\in \mathbb{N}$ determines the degree of sparsity and $\epsilon\in\mathbb{R}^+$ prevents the all-zero solution.
A solution to the above problem yields a sparse bipartite graph which corresponds to the basis vectors of the null space specified by the patterns in the training set. In other words, the inherent structure of the patterns is captured in terms of $m$ linear constraints on the entries of the patterns $x^\mu$ in the training set. It can therefore be described by Figure \ref{single_level_net} with a connectivity matrix $W \in \mb{R}^{m \times n}$ such that $Wx^\mu = 0$ for all $\mu = 1,\dots,C$. 

In the recall phase, the neural network is fed with noisy inputs.
%inputs. The recall process can be described using the graph in
 A possibly noisy version of an input pattern is initialized as the states of the pattern neurons $x_1,x_2,\dots,x_n$. Here, we assume that the noise is integer valued and additive\footnote{It must be mentioned that neural states below $0$ and above $S$ will be set to $0$ and $S$, respectively.}. In formula, we have $y = W( x^\mu + z) = Wz$ where $z$ is the noise added to pattern $x^\mu$ and we used the fact that $Wx^\mu = \textbf{0}$. Therefore, one can use $y = Wz$ to eliminate the input noise $z$. Consequently, we are searching an algorithm that can provably eliminate the effect of noise and return the correct pattern.
\begin{rem}
 A solution in the learning/recall phase is acceptable only if it can be found by simple operations at neurons. 
\end{rem}

\begin{figure}[t]
\begin{center}
\includegraphics[width=.25\textwidth]{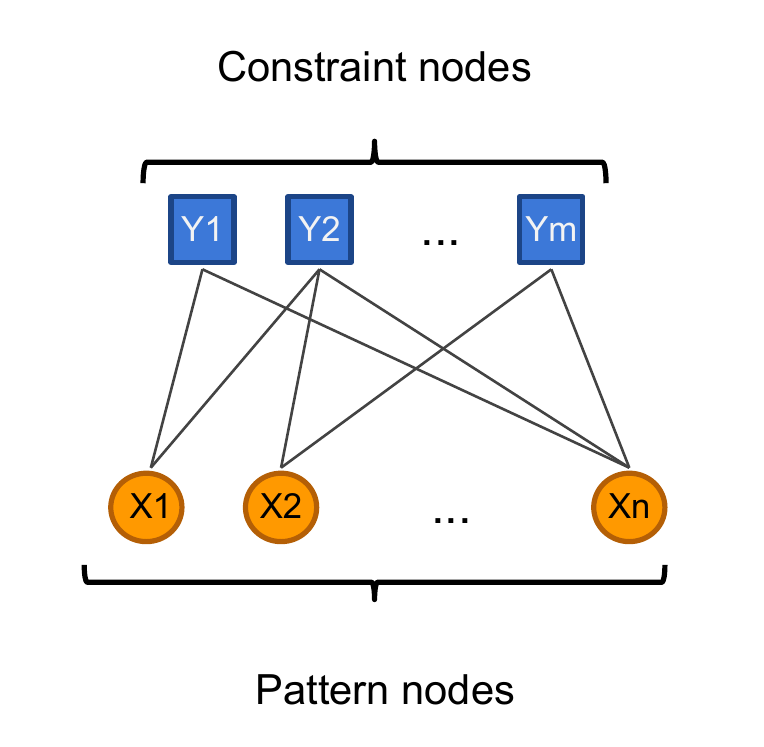}
\end{center}
\begin{center}
\vspace{-0.5cm}
\caption{A bipartite graph that represents the constraints on the training set.\label{single_level_net}}
\end{center}
\vspace{-1cm} 
\end{figure}

Before presenting our solution, we briefly overview the relation between the previous works and the one presented in this paper.
\subsection{Related Works}\label{sec:related}
Designing a neural network capable of learning a set of patterns and recalling them later in presence of noise has been an active topic of research for the past three decades. Inspired by the Hebbian learning rule \cite{hebb}, Hopfield in his seminal work \cite{hopfield} introduced the Hopfield network: an auto-associative neural mechanism of size $n$ with binary state neurons in which patterns are assumed to be binary vectors of length $n$. The capacity of a Hopfield network under vanishing bit error probability was later shown to be $0.13n$ by Amit et al. \cite{amit}. Later on, McEliece et al. proved that the capacity of Hopfield networks under vanishing block error probability requirement is $O(n/\log(n))$ \cite{mceliece}. Similar results were obtained for sparse regular neural network in \cite{Komlos}. It is also known that the capacity of neural associative memories could be enhanced if the patterns are \emph{sparse} in the sense that at any time instant many of the neurons are silent \cite{hertz}. However, even these schemes fail when required to correct a fair amount of erroneous bits as the information retrieval is not better compared to that of normal networks.

In addition to neural networks capable of learning patterns gradually, in \cite{venkatesh}, the authors calculate the weight matrix offline (as opposed to gradual learning) using the pseudo-inverse rule \cite{hertz} which in return help them improve the capacity of a Hopfield network to $n/2$ \emph{random patterns} with the ability of \textit{one bit} error correction. 

Due to the low capacity of Hopfield networks, extension of associative memories to non-binary neural models has also been explored in the past. Hopfield addressed the case of continuous neurons and showed that similar to the binary case, neurons with states between $-1$ and $1$ can memorize a set of random patterns, albeit with less capacity \cite{hopfield_non_binary}. In \cite{Jankowski} the authors investigated a multi-state complex-valued neural associative memories for which the estimated capacity is $C < 0.15 n$. Under the same model but using a different learning method, Muezzinoglu et al. \cite{Muezzinoglu1} showed that the capacity can be increased to $C = n$. However the complexity of the weight computation mechanism is prohibitive. 
To overcome this drawback, a Modified Gradient Descent learning Rule (MGDR) was devised in \cite{Lee}. 

Given that even very complex offline learning methods can not improve the capacity of binary or multi-sate Hopfield networks, a line of recent work has made considerable efforts to exploit the inherent structure of the patterns in order to increase both capacity and error correction capabilities. Such methods either make use of higher order correlations of patterns or focus merely on those patterns that have some sort of redundancy. As a result, they differ from previous methods for which every possible random set of patterns was considered. Pioneering this prospect, Berrou and Gripon \cite{gripon_sparse} achieved considerable improvements in the pattern retrieval capacity of Hopfield networks, by utilizing clique-based coding. In some cases, the proposed approach results in capacities of around $30n$, which is much larger than $O(n/\log(n))$ in other methods. In \cite{SKGS}, the authors used low correlation sequences similar to those employed in CDMA communications to increase the storage capacity of Hopfield networks to $n$ without requiring any separate decoding stage. 

In contrast to the pairwise correlation of the Hopfield model \cite{hopfield}, Peretto et al. \cite{peretto} deployed \emph{higher order} neural models: the state of the neurons not only depends on the state of their neighbors, but also on the correlation among them. Under this model, they showed that the storage capacity of a higher-order Hopfield network can be improved to $C=O(n^{p-2})$, where $p$ is the degree of correlation considered. The main drawback of this model was again the huge computational complexity required in the learning phase. To address this difficulty while being able to capture higher-order correlations, a bipartite graph inspired from iterative coding theory was introduced in \cite{KSS}. Under the assumptions that the bipartite graph is known, sparse, and expander, the proposed algorithm increased the pattern retrieval capacity to $C=O(a^n)$, for some $a > 1$. The main drawbacks in the proposed approach is the lack of a learning algorithm as well as the assumption that the weight matrix should be an expander. The sparsity criterion on the other hand, as it was noted by the authors, is necessary in the recall phase and biologically more meaningful.

In this paper, we focus on solving the above two problems in \cite{KSS}. We start by proposing an iterative learning algorithm that identifies a \emph{sparse} weight matrix $W$. The weight matrix $W$ should satisfy a set of linear constraints $Wx^\mu = 0$ for all the patterns $x^\mu$ in the training data set, where $\mu=1,\dots,C$. We then propose a novel network architecture which eliminates the need for the expansion criteria while achieving better performance than the error correction algorithm proposed in \cite{KSS}. 

Constructing a factor-graph model for neural associative memory has been also addressed in \cite{braunstein1}. However, there, the authors propose a general message-passing algorithm to memorize any set of random patterns while we focus on memorizing patterns belonging to subspaces with sparsity in mind as well. The difference would again be apparent in the pattern retrieval capacity (linear vs. exponential in network size).

Learning linear constraints by a neural network is hardly a new topic as one can learn a matrix orthogonal to a set of patterns in the training set (i.e., $Wx^\mu = 0$) using simple neural learning rules (we refer the interested readers to \cite{xu} and \cite{oja}). However, to the best of our knowledge, finding such a matrix subject to the sparsity constraints has not been investigated before. This problem can also be regarded as an instance of compressed sensing \cite{candes}, in which the measurement matrix is given by the big patterns matrix $\mc{X}_{C \times n}$ and the set of measurements are the constraints we look to satisfy, denoted by the tall vector $b$, which for simplicity reasons we assume to be all zero. Thus, we are interested in finding a sparse vector $w$ such that $\mc{X}w = 0$.

Nevertheless, many decoders proposed in this area are very complicated and cannot be implemented by a neural network using simple neuron operations. Some exceptions are \cite{donoho_amp} and \cite{tropp} from which we derive our learning algorithm.

%  There are lots of methods in compressive sensing to solve this problem. LASSO-based approaches look specially interesting as the problem formulation looks very similar to neural learning problems \cite{osborne}. However, we require such algorithms to be simple enough to be implemented by neural networks. That limits our span to simple iterative algorithms such as \cite{donoho_amp}, \cite{tropp}.

\section{Learning Algorithm}\label{sec:learning}
%Suppose we have $C$ vectors of length $n$ with non-negative integer entries, where these patterns form a subspace of dimension $k < n$. We would like to memorize these patterns by finding a set of non-zero vectors $w_1,\dots,w_m \in \mb{R}^n$ that are orthogonal to the set of given patterns. Furthermore, we are interested in rather sparse vectors. Putting the training patterns in a matrix $\mathcal{X}_{C \times n}$ and focusing on one such vector $w$, we can formulate the problem as: 
%
%\begin{subequations}\label{main_problem}
%\beq\label{main_problem_obj}
%\min \Vert \mathcal{X}.w \Vert_2
%\eeq
%subject to
%\beq\label{main_problem_const1}
%\Vert w \Vert_0 \leq q
%\eeq
%and
%\beq\label{main_problem_const2}
%\Vert w \Vert_2^2 \geq \epsilon
%\eeq
%\end{subequations}
%where $q$ determines the degree of sparsity and the constraint (\ref{main_problem_const2}) ensures that the algorithm will not converge to the all-zero solution.

We are interested in an iterative algorithm that is simple enough to be implemented by a network of neurons. Therefore, we first relax (\ref{main_problem}) as follows:
\beq\label{main_problem_modified}
\min \Vert \mathcal{X}\cdot w \Vert_2 - \lambda (\Vert w \Vert_2^2 - \epsilon) + \gamma(g(w) - q').
\eeq
In the above problem, we have approximated the constraint $\Vert w \Vert_0 \leq q$ with $g(w) \leq q'$ since $\Vert . \Vert_0$ is not a well-behaved function. The function $g(w)$ is chosen such that it favors sparsity. For instance one can pick $g(w)$ to be $\Vert . \Vert_1$, which leads to $\ell_1$-norm minimizations. In this paper, we consider the function
\begin{displaymath}
g(w) = \sum_{i = 1}^n \tanh(\sigma w_i^2)
\end{displaymath}
where $\sigma$ is chosen appropriately. By calculating the derivative of the objective function and primal-dual optimization techniques we obtain the following iterative algorithm for (\ref{main_problem_modified}):

%In order to find the solution to problem (\ref{main_problem_modified}), by calculating the derivative of the objective function and primal-dual optimization techniques we propose the following algorithm:
\begin{subequations}\label{proposed_algo_old}

\beq\label{proposed_algo_old_y}
y(t) = \frac{\mathcal{X}\cdot w(t)}{\Vert \mathcal{X} \Vert_2}
\eeq
\beq\label{proposed_algo_old_w}
w(t+1) = (1+2\lambda_t)w(t)-2\alpha_t \frac{\mathcal{X}^\top y(t)}{\Vert \mathcal{X} \Vert_2}- \gamma_t \nabla g(w)
\eeq
\beq\label{proposed_algo_old_lambda}
\lambda_{t+1} = \left[\lambda_t + \delta(\epsilon-\Vert w \Vert_2^2) \right]
\eeq
\beq\label{proposed_algo_old_gamma}
\gamma_{t+1} = \left[\gamma_t + \delta(g(w) - q') \right]
\eeq
\end{subequations}
where $t$ denotes the iteration number, $\mathcal{X}^\top$ is the transpose of matrix $\mathcal{X}$, $\delta$ and $\alpha_t$ are small step sizes and $[\cdot]_+$ denotes $\max(\cdot,0)$. 

For our choice of $g(w)$, the $i^{th}$ entry of the function $f(w) = \nabla g(w)$, denoted by $f_i(w)$ reduces to $2 \sigma w_i(1-\tanh(\sigma w_i^2)^2)$. For very small values of $w_i$, $f_i(w) \simeq w_i$ and for large values of $w_i$, $f_i(w) \simeq 0 $. Therefore, by looking at (\ref{proposed_algo_old_w}) we see that the last term is pushing small values in $w(t+1)$ towards zero while leaving the larger values intact. Therefore, we remove the last term completely and enforce small entries to zero in each update which in turn enforces sparsity. The final iterative learning procedure is shown in Algorithm~\ref{algo:learning}.
\begin{algorithm}[t]
\caption{Iterative Learning}
%Finding Target with Entropy plus the logarithm of the diameter of the constraint graph Search Queries}
\label{algo:learning}
\begin{algorithmic}
\REQUIRE{ pattern matrix $\mc{X}$, stopping point $p$.}
\ENSURE{$w$}
\WHILE{ $\Vert y(t) \Vert_{\max}>p$ }
\STATE Compute $y(t) = \frac{\mathcal{X}\cdot w(t)}{\Vert \mathcal{X} \Vert_2}.$
\STATE Update $w(t+1) = \eta \left((1+2\lambda_t)w(t)-2\alpha_t \frac{\mathcal{X}^\top y(t)}{\Vert \mathcal{X} \Vert_2} \right)_{\theta_t}$.
\STATE Update $\lambda_{t+1} = \left[\lambda_t + \delta(\epsilon-\Vert w \Vert_2^2) \right]$.
\STATE $t \leftarrow t+1$.
\ENDWHILE
\end{algorithmic}
\end{algorithm}

Here, $\theta_t$ is a positive threshold at iteration $t$ and $\eta(.)_{\theta_t}$ is the \emph{point-wise} soft-thresholding function given below:
\beq\label{soft_threshold_func}
\eta(u)_\theta = \left\{ \begin{array}{ll}
  u & \mbox{if $u > \theta$},\\
  u & \mbox{if $u < -\theta$},\\
  0 & \mbox{otherwise}.\end{array} \right.
\eeq
\begin{rem}
the above choice of soft-theresholding function is very similar to the one selected by Donoho et al. in \cite{donoho_amp} in order to recover a sparse signal from a set of measurements. The authors prove that their choice of soft-threshold function results in \emph{optimal} sparsity-undersampling trade-off.
\end{rem}

The next theorem derives the necessary conditions on $\alpha_t$, $\lambda_t$ and $\theta_t$ such that Algorithm \ref{algo:learning} converges to a sparse solution.

\begin{theorem}
If $\theta_t \rightarrow 0$ as $t \rightarrow \infty$ and if $\lambda_t$ is bounded above by $a_{\min}/(a_{\max}-a_{\min})$, then there is a proper choice of $\alpha_t$ in every iteration $t$ that ensures constant decrease in the objective function $\Vert \mathcal{X}.w(t)\Vert_{\max}$. Here $a_{\min} = \min_{\mu} \Vert x^\mu \Vert^2 / \Vert \mathcal{X} \Vert^2$ and $a_{\max} = \max_{\mu} \Vert x^\mu \Vert^2 / \Vert \mathcal{X} \Vert^2$. For $\lambda_t = 0$, i.e. $\Vert w(t) \Vert_2 \geq \epsilon$, picking $0 < \alpha_t<1$ ensures gradual convergence.
\end{theorem}

\begin{proof}[Sketch of the proof] Let $E(t) = \Vert y(t) \Vert_{\max}$. We would like 
Let $E(t) = \Vert y(t) \Vert_{\max}$. We would like to show that $E(t+1) < E(t)$ for all iterations $t$. To this end, let us denote $(1+2\lambda_t)w(t)-2\alpha_t \frac{\mathcal{X}^T y}{\Vert \mathcal{X} \Vert_2}$ by $w'(t)$. Furthermore, let the function $\chi(u;\theta_t)$ be $u-\eta(u)_{\theta_t}$. Rewriting the second step of algorithm (\ref{algo:learning}) we will have:
\beq
w(t+1) = w'(t) - \chi(w'(t);\theta_t)
\eeq

Now we have 
\beqa\label{proof1}
E(t+1) &=& \Vert y(t+1) \Vert_{\max} = \Vert \frac{\mathcal{X}.w(t+1)}{\Vert \mathcal{X} \Vert_2} \Vert_{\max}\nonumber\\
%&=& \Vert \frac{\mathcal{X}.w'(t)}{\Vert \mathcal{X} \Vert} - \frac{\mathcal{X}.\chi(w'(t);\theta_t)}{\Vert \mathcal{X} \Vert_2}\Vert_{\max}\nonumber\\
&\leq& \Vert \frac{\mathcal{X}.w'(t)}{\Vert \mathcal{X} \Vert_2} \Vert_{\max} + \Vert \frac{\mathcal{X}.\chi(w'(t);\theta_t)}{\Vert \mathcal{X} \Vert_2}\Vert_{\max}\nonumber\\
%&\leq& \Vert \frac{\mathcal{X}.w'(t)}{\Vert \mathcal{X} \Vert_2} \Vert_{\max} + \frac{\Vert \mathcal{X} \Vert_{\max} \Vert \chi(w'(t);\theta_t)\Vert_{\max}}{\Vert \mathcal{X} \Vert_2}\nonumber\\
&\leq& \Vert \frac{\mathcal{X}.w'(t)}{\Vert \mathcal{X} \Vert_2} \Vert_{\max} + \theta_t \frac{\Vert \mathcal{X} \Vert_{\max}}{\Vert \mathcal{X} \Vert_2}\nonumber\\
&\leq& \Vert \frac{\mathcal{X}.w'(t)}{\Vert \mathcal{X} \Vert_2} \Vert_{\max} + \theta_t 
\eeqa
where the last inequality follows because $\Vert \mathcal{X} \Vert_{\max} \leq \Vert \mathcal{X} \Vert_2$. Now expanding $\frac{\mathcal{X}.w'(t)}{\Vert \mathcal{X} \Vert}$ we will get 
\beqa\label{eq_expand1}
\frac{\mathcal{X}.w'(t)}{\Vert \mathcal{X} \Vert_2} &=& (1+2\lambda_t)y(t)-2\alpha_t \frac{\mathcal{X} \mathcal{X}^T y}{\Vert \mathcal{X} \Vert_2^2}\nonumber \\ 
&=& \left[(1+2\lambda_t)I_{C\times C} -2\alpha_t \frac{\mathcal{X} \mathcal{X}^T}{\Vert \mathcal{X} \Vert_2^2} \right] y(t)
\eeqa
Denoting the matrix $(1+2\lambda_t)I_{C\times C} -2\alpha_t \frac{\mathcal{X} \mathcal{X}^T}{\Vert \mathcal{X} \Vert_2^2}$ by $D_t$, we can further simplify inequality (\ref{proof1}):
\beqa\label{proof2}
E(t+1) &\leq& \Vert D_t y(t) \Vert_{\max} + \theta_t \nonumber \\
&\leq& \Vert D_t\Vert_{\max} \Vert y(t) \Vert_{\max} + \theta_t \nonumber \\
&=& \Vert D_t\Vert_{\max} E(t) + \theta_t 
\eeqa
Where $D(t) = (1+2\lambda_t)I_{C\times C} -2\alpha_t \frac{\mathcal{X} \mathcal{X}^T}{\Vert \mathcal{X} \Vert_2^2}$. Therefore, if we set $\theta_t = 1/t$ (i.e. $\theta_t \rightarrow 0$ as $t \rightarrow \infty$) and ensuring $\Vert D_t\Vert_{\max} < 1$ we get $E(t+1) < E(t)$. The second requirement requires that $|D_{ij}(t)| <1$ for all elements $(i,j)$ of $D_t$. Therefore, by letting $A = XX^T$ we must have the following relationship for diagonal elements:
\beqa\label{alpha_amp_ineq1}
|1+2\lambda_t - 2\alpha_t \frac{A_{ii}}{\Vert A\Vert_2}| < 1 
\eeqa
which yields
\bed
\lambda_t < \alpha_t \frac{A_{ii}}{\Vert A\Vert_2} < 1+\lambda_t,\ \forall i=1,\dots C
\eed
Since $\lambda_t \geq 0$ for all $t$ and $0 < A_{ii} \leq \Vert A \Vert_2$, the right hand side of the above inequality is satisfied if $\alpha_t < 1$. The left-hand side is satisfied for $\alpha_t > \lambda_t /a_{\min}$, where $a_{\min} = \min_i A_{ii}/\Vert A\Vert_2$. Therefore, if $\lambda_t \leq a_{\min}/(a_{\max}-a_{\min})$ there exists and $\alpha_t$ ensuring $\Vert D \Vert_{max} < 1$. If $\lambda_t = 0$, this is simply equivalent to having $0 < \alpha < 1$.

\end{proof}

\section{Multi-level Network Architecture}\label{sec:multi}
In the previous section, we discussed the details of a simple iterative learning algorithm which yields rather sparse graphs. Now in the recall phase, we propose a network structure together with a simple error correction algorithm (similar to the one in \cite{KSS}) to achieve good block error rates in response to noisy input patterns. The suggested network architecture is shown in Figure~\ref{multi_level_net}.
To make the description clear and simple we only concentrate on a two-level neural network. However, the generalization of this idea is trivial and left to the reader. 

\begin{figure}[t]
\begin{center}
\includegraphics[width=.35\textwidth]{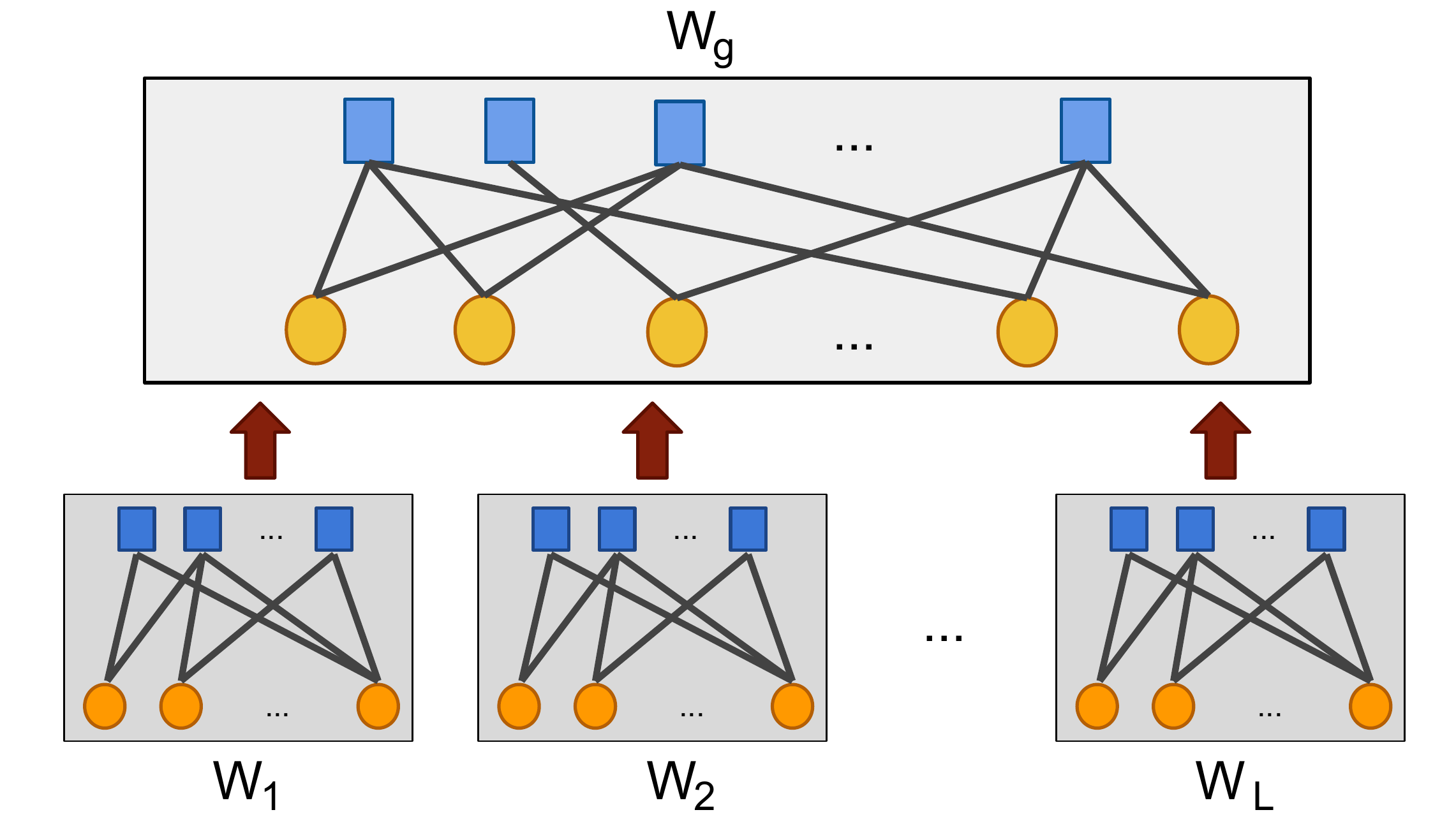}
\end{center}
\begin{center}
\vspace{-0.5cm}
\caption{A two-level error correcting neural network. \label{multi_level_net}}
\end{center}
\vspace{-0.5cm}
\end{figure}

The proposed approach is in contrast to the one in suggested in \cite{KSS} where the authors exploit a single-level neural network with a sparse and \emph{expander} connectivity graph to correct at least two initial input errors. However, enforcing expansion on connectivity graphs in a gradual neural learning algorithm is extremely difficult, specially when the algorithm is required to be very simple %??? Reference 
Therefore, we use the learning algorithm explained above, which yields a \emph{rather} sparse and not necessarily expander graph, and improve the error correction capabilities by modifying the network structure and error correcting algorithm.

\begin{algorithm}[t]
\caption{Error Correction}
%Finding Target with Entropy plus the logarithm of the diameter of the constraint graph Search Queries}
\label{algo:correction}
\begin{algorithmic}[1]
\REQUIRE{ pattern matrix $\mc{X}$, threshold $\varphi$, iteration $t_{\max}$}
\ENSURE{$x_1,x_2,\dots,x_n$}
\FOR{$t = 1 \to t_{\max}$} 
\STATE \textit{Forward iteration:} Calculate the weighted input sum $ h_i = \sum_{j=1}^n W_{ij} x_j,$ for each neuron $y_i$ and set:
\[ y_i = \left\{
\begin{array}{cc}
1, & h_i < 0\\
0, & h_i = 0\\
-1, & \hbox{otherwise}\footnotemark
\end{array} \right. .
\]
\STATE \textit{Backward iteration:} Each neuron $x_j$ computes
\[ g_j = \frac{\sum_{i=1}^m W_{ij} y_i }{\sum_{i = 1}^m |W_{ij}|}. \]
\STATE Update the state of each pattern neuron $j$ according to $x_j = x_j + \hbox{sgn}(g_j)$
only if $|g_j| > \varphi$.
 \STATE $t \gets t + 1$
\ENDFOR
\end{algorithmic}
Note that in practice, we replace the condition $h_i = 0$ and $h_i>0$ with $|h_i| < \varepsilon$ and $h_i > \varepsilon$ for some small positive number $\varepsilon$.
\end{algorithm}
\footnotetext{Note that although we do not allow neurons to have negative outputs, the set of outputs $\{-1,0,1\}$ can be easily implemented by sending $\{0,1,2\}$ and shift the response in the pattern neurons in each iteration. The shifting can be done by modifying the firing threshold for each neuron.}

The idea behind this new architecture is that we divide the input pattern of size $n$ into $L$ sub-patterns of length $n/L$. Now we feed each sub-pattern to a neural network which enforces $m$ constraints\footnote{The number of constraints for different networks can vary. For simplicity of notifications we assume equal sizes.} on the sub-pattern in order to correct the input errors. The local networks in the first level and the global network in the second level use Algorithm~\ref{algo:correction}, which is a variant of the "bit-flipping" method proposed in \cite{KSS}, to correct the errors. Intuitively, if the states of the pattern neurons $x_i$ correspond to a
pattern from $\mathcal{X}$ (i.e., the noise-free case), then for
 all $i=1,\ldots,m$ we have $y_i=0$.
The quantity $g_j$ can be interpreted as feedback to pattern neuron $x_j$ from the constraint neurons. Hence, the sign of $g_j$ provides an indication of the sign of the noise that affects $x_j$, and $|g_j|$
indicates the confidence level in the decision regarding the sign of
the noise.

\begin{theorem}
Algorithm~\ref{algo:correction} can correct a single error in the input pattern with high probability if $\varphi$ is chosen large enough.
\end{theorem}
\begin{proof}
In the case of a single error, we are sure that the corrupted node will always be updated towards the correct direction. For simplicity, let's assume the first pattern neuron is the noisy one. Furthermore, let $z = \{1,\dots,0\}$ be the noise vector. Denoting the $i^{th}$ column of the weight matrix by $W_i$, we will have $y = \hbox{sign}(W_1)$. Then in algorithm \ref{algo:correction} $g_1 = 1 > \varphi$. This means that the noisy node gets updated towards the correct direction.

Therefore, the only source of error would be a correct node gets updated mistakenly. Let $P_{x_i}$ denote the probability that a correct pattern neuron $x_i$ gets updated. This happens if $|g_{x_i}|> \varphi$. For $\varphi = 1$, this is equivalent to having $W_i\cdot \hbox{sign}(z_1 W_1) = \Vert W_i \Vert_0$. Note that $W_i\cdot \hbox{sign}(W_1) < \Vert W_i \Vert_0$ in cases that the neighborhood of $x_i$ is different from the neighborhood of $x_1$ among the constraint nodes. More specifically, in the case that $\mc{N}(x_i) \neq \mc{N}(x_1)$, there are non-zero entries in $W_i$ while $W_1$ is zero and vice-versa. Therefore, letting $P'_{x_i}$ being the probability of $\mc{N}(x_i) \neq \mc{N}(x_1)$, we note that
\bed
P_{x_i} \leq P'_{x_i}
\eed
Therefore, to get an upper bound on $P_{x_i}$, we bound $P'_{x_i}$.

Let $\lambda_i$ be the fraction of pattern neurons with degree $i$, $\bar{d} = \sum_i \lambda_i d_i$ be the average degree of pattern neurons and finally $d_{min}$ be the minimum degree of pattern neurons. Then, we know that a noisy pattern neuron is connected to $\bar{d}$ constraint neurons on average. Therefore, the probability of $x_i$ and $x_1$ share exactly the same neighborhood would be:
\beq
P'_{x_i} = \left(\frac{\bar{d}}{m} \right)^{d_{x_i}}
\eeq

Taking the average over the pattern neurons, we have
\beqa
P'_e &=& \hbox{Pr}\{x \in \mc{C}_t \}\mb{E}_{d_{x_i}}\{P'_{x_i}\} \nonumber \\
&=& (1-\frac{1}{n}) \lambda(\frac{\bar{d}}{m})\nonumber \\
&=& \lambda(\frac{\bar{d}}{m})
\eeqa
where $\mc{C}_t$ is the set of correct nodes at iteration $t$ and $\lambda(x) = \sum_i \lambda_i x^i$.

Therefore, the probability of correcting one noisy input, $P_c = 1-P_e \geq 1-P'_e$ would be
\beqa\label{probability_correction}
P_c &\geq& 1- \lambda(\frac{\bar{d}}{m}) \nonumber \\
&\geq& 1- \left(\frac{\bar{d}}{m} \right)^{d_{min}}
\eeqa
\end{proof}

Given that each local network is able to correct one pattern, $L$ such networks can correct $L$ input errors \emph{if} they are separated such that only one error appears in the input of each local network. Otherwise, there would be a probability that the network could not handle the errors. In that case, we feed the overall pattern of length $n$ to the second layer with the connectivity matrix $W_g$, which enforces $m_g$ \emph{global} constraints. And since the probability of correcting two erroneous nodes increases with the input size, we expect to have a better error correction probability in the second layer. Therefore, using this simple scheme we expect to gain a lot in correcting errors in the patterns. In the next section, we provide simulation results which confirm our expectations and show that the block error rate can be improved by a factor of $100$ in some cases.

\subsection{Some remarks}
First of all, one should note that the above method only works if there is some redundancy at the global level as well. If the set of weight matrices $W_1,\dots,W_L$ define completely separate sub-spaces in the $n/L$-dimensional space, then for sure we gain nothing using this method. 

%However, if there are redundancies in the sub-spaces, we hope that the global constraints become helpful in the error correction process. A simple extreme is the case of repetition, where the subspaces look exactly the same. In this case, different sub-patterns should be equal to one another. 

Secondly, there is no need to have the dimension of the subspaces to be equal to each other. We can have different lengths for the sub-patterns belonging to each subspace and different number of constraints for that particular sub-space. This gives us more degree of freedom as well since we can spend some time to find the optimal length of each sub-pattern for a particular training data set. 

Thirdly, the number of constraints for the second layer affects the gain one obtains in the error performance. Intuitively, if the number of global constraints is large, we are enforcing more constraints so we expect obtaining a better error performance. We can think of determining the number even adaptively, i.e. if the error performance that we are getting is unacceptable, we can look deeper in patterns to identify their internal structure by searching for more constraints. This would be a subject of our future research.

\section{Pattern Retrieval Capacity}\label{sec:capacity}
the following theorem will prove that the proposed neural architecture is capable of memorizing an exponential number of patterns.
\begin{theorem}
Let $\mc{X}$ be a $C \times n$ matrix, formed by $C$ vectors of length $n$ with non-negative integers entries between $0$ and $S-1$. Furthermore, let $k_g = rn$ for some $0<r<1$. Then, there exists a set of such vectors for which $C = a^{rn}$, with $a > 1$, and $\hbox{rank}(\mc{X}) = k_g <n$, and such that they can be memorized by the neural network given in figure \ref{multi_level_net}.
\end{theorem}

\begin{proof}
The proof is based on construction: we construct a data set $\mc{X}$ with the required properties, namely the entries of patterns should be non-negative, patterns should belong to a subspace of dimension $k_g <n$ and each sub-pattern of length $n/L$ belongs to a subspace of dimension $k < n/L$.

To start, consider a matrix $G \in \mb{R}^{k_g \times n}$ with rank $k_g$ and $k_g = rn$, with $0 < r < 1$. Let the entries of $G$ be non-negative integers, between $0$ and $\gamma-1$, with $\gamma \geq 2$. Furthermore, let $G_1,\dots,G_L$ be the $l$ sub-matrices of $G$, where $G_i$ comprises of the columns $1+(i-1)L$ to $iL$ of $G$. Finally, assume \emph{in each sub-matrix} we have exactly $k$ non-zero rows with $k<n/L$. 

We start constructing the patterns in the data set as follows: consider a random vector $u^\mu \in \mb{R}^{k_g}$ with integer-valued-entries between $0$ and $\upsilon-1$, where $\upsilon \geq 2$.  We set the pattern $x^\mu \in \mc{X}$ to be $x^\mu = u^\mu \cdot G$, \emph{if} all the entries of $x^\mu$ are between $0$ and $S-1$. Obviously, since both $u^\mu$ and $G$ have only non-negative entries, all entries in $x^\mu$ are non-negative. Therefore, it is the $S-1$ upper bound that we have to worry about. 

The $j^{th}$ entry in $x^\mu$ is equal to $x_j^\mu = u^\mu \cdot \mb{g}_j$, where $\mb{g}_j$ is the $j^{th}$ column of $G$. Suppose $\mb{g}_j$ has $d_j$ non-zero elements. Then, we have: 
\bed
x_j^\mu = u^\mu \cdot \mb{g}_j \leq d_j (\gamma-1) (\upsilon-1)
\eed

Therefore, denoting $d^* = \max_{j} d_j$, we could choose $\gamma$, $\upsilon$ and $d^*$ such that
\beq\label{capaci_inequal}
S-1 \geq d^* (\gamma-1) (\upsilon-1)
\eeq
to ensure all entries of $x^\mu$ are less than $S$. 

As a result, since there are $\upsilon^{k_g}$ vectors $u$ with integer entries between $0$ and $\upsilon -1$, we will have $\upsilon^{k_g} = \upsilon^{rn}$ patterns forming $\mc{X}$. Which means $C = \upsilon^{rn}$, which would be an exponential number in $n$ if $\upsilon \geq 2$. 
\end{proof}

\section{Simulation Results}\label{sec:simul}
We have simulated the proposed learning algorithm in the multi-level architecture to investigate the block error rate of the suggested approach and the gain we obtain in error rates by adding a second level. We constructed $4$ local networks, each with $n/4$ pattern and $m$ constraint nodes. 

\subsection{Learning Phase}
We generated a sample data set of $C = 10000$ patterns of length $n$ where each block of $n/4$ belonged to a subspace of dimension $k < n/4$. Note that $C$ can be an exponential number in $n$. However, we selected $C = 10000$ as an example to show the performance of the algorithm because even for small values of $k$, and exponential number in $k$ will become too large to handle numerically. The result of the learning algorithm is four different local connectivity matrices $W_1,\dots,W_4$ as well as a global weight matrix $W_g$. The number of local constraints was $m = n/4 -k$ and the number of global constraints was $m_g = n-k_g$, where $k_g$ is dimension of the subspace for overall pattern. The learning steps are done until $99\%$ of the patterns in the training set converged. Table \ref{simul_param_learn} summarizes other simulation parameters.
\begin{table}[h]
\label{simul_param_learn}
\caption{Simulation parameters}
\begin{tabular}{ |c| c| c| c|c|c| }
  \hline                       
  Parameter & $\delta$ & $\theta_t$ & $\alpha_t$ (when $\lambda_t\neq 0$) & $\epsilon$ & $p$\\ \hline
  Value & $10$ & $0.25/t$ & $ \min(\frac{\lambda_t}{a_{\min}},1+\lambda_t)$ & $0.01$ & $\frac{0.01}{\Vert X \Vert_2} $ \\
  \hline  
\end{tabular}
\end{table}
For cases where $\lambda_t = 0$, $\alpha_t$ was fixed to $0.49$. 

Table \ref{table_learn_itr_sparsity} shows the average number of iterations executed before convergence is reached for different constraint nodes at the local and global level. It also gives the average sparisty ratio for the columns of matrix $W$. The sparsity ratio is defined as $\rho = \kappa/n$, where $\kappa$ is the number of non-zero elements. From the figure one notices that as $n$ increases, the vectors become sparser.

\begin{table}[h]
\label{table_learn_itr_sparsity}
\caption{Average number of convergence iterations and sparsity in the local and global networks for $n=400$}
\begin{center}
\begin{tabular}{c|c| c| c|c| }
 \cline{2-5}   
 & \multicolumn{2}{|c|}{Sparsity Ratio} & \multicolumn{2}{|c|}{Convergence Rate } \\ \cline{2-5}   
 & $k_g = 100$ & $k_g=200$ & $k_g=100$ & $k_g=200$\\ \hline
 \multicolumn{1}{|c|}{Local} & $0.28$ & $0.32$ & $4808$ & $5064$ \\ \hline 
 \multicolumn{1}{|c|}{Global} & $0.22$ & $0.26$ & $14426$ & $33206$\\
 \hline 
\end{tabular}
\end{center}
\end{table}

%\begin{figure}[t]
%\begin{center}
%\includegraphics[width=.4\textwidth]{amp_learn_itr_90_300_600.png}
%\end{center}
%\begin{center}
%\caption{Convergence rate for different constraint nodes and different values of network sizes: $n = 360,1200$.\label{amp_learn_itr}}
%\end{center}
%\end{figure}

%\begin{figure}[t]
%\begin{center}
%\includegraphics[width=.45\textwidth]{amp_sparisty_percentage_90_300_600.png}
%\end{center}
%\begin{center}
%\caption{The percentage of trials with the specified sparsity measure and different values of network sizes: $n = 360,1200$. The sparsity measure is defined as $\rho = \kappa/n$, where $\kappa$ is the number of non-zero elements.\label{amp_sparisty_percentage}}
%\end{center}
%\end{figure}

%Figure \ref{amp_learning_cost} shows the MSE, given by the $\Vert y(t) \Vert_{\max}$ in algorithm \ref{algo:learning}, in each iteration for a sample trial for different values of $n$. Overall, the MSE is decreasing. 
%\begin{figure}[t]
%\begin{center}
%\includegraphics[width=.45\textwidth]{amp_learning_cost_90_300_600.png}
%\end{center}
%\begin{center}
%\vspace{-0.5cm}
%\caption{MSE as a function of iteration for a random trial node and three different values of $n$.\label{amp_learning_cost}}
%\end{center}
%\vspace{-0.5cm}
%\end{figure}
\subsection{Recall Phase}
For the recall phase, in each trial we pick a pattern randomly from the training set, corrupt a given number of its symbols with $\pm 1$ noise and use the suggested algorithm to correct the errors. As mentioned earlier, the errors are corrected first at the local and the at the global level. When finished, we compare the output of the first and the second level with the original (uncorrupted) pattern $x$. A pattern error is declared if the output does not match at each stage. Table \ref{simul_param_recall} shows the simulation parameters in the recall phase. 

\begin{table}[h]
\begin{center}
\label{simul_param_recall}
\caption{Simulation parameters}
\begin{tabular}{ |c| c| c|c| }
  \hline                       
  Parameter & $\varphi$ & $t_{\max}$ & $\varepsilon$ \\ \hline
  Value & $0.8$ & $ 20 \Vert z \Vert_0 $ & $0.01$ \\
  \hline  
\end{tabular}
\end{center}
\end{table}

Figure \ref{PER_full_redund} illustrates the pattern error rates $n = 400$ with two different values of $k_g=100$ and $k_g = 200$. The results are also compared to that of the bit-flipping algorithm in \cite{KSS} to show the improved performance of the proposed algorithm. As one can see, having a larger number of constraints at the global level, i.e. having a smaller $k_g$, will result in better pattern error rates at the end of the second stage. Furthermore, note that since we stop the learning after $99\%$ of the patterns had learned, it is natural to see some recall errors even for $1$ initial erroneous node. 
\begin{figure}[h]
\begin{center}
\includegraphics[width=.36\textwidth]{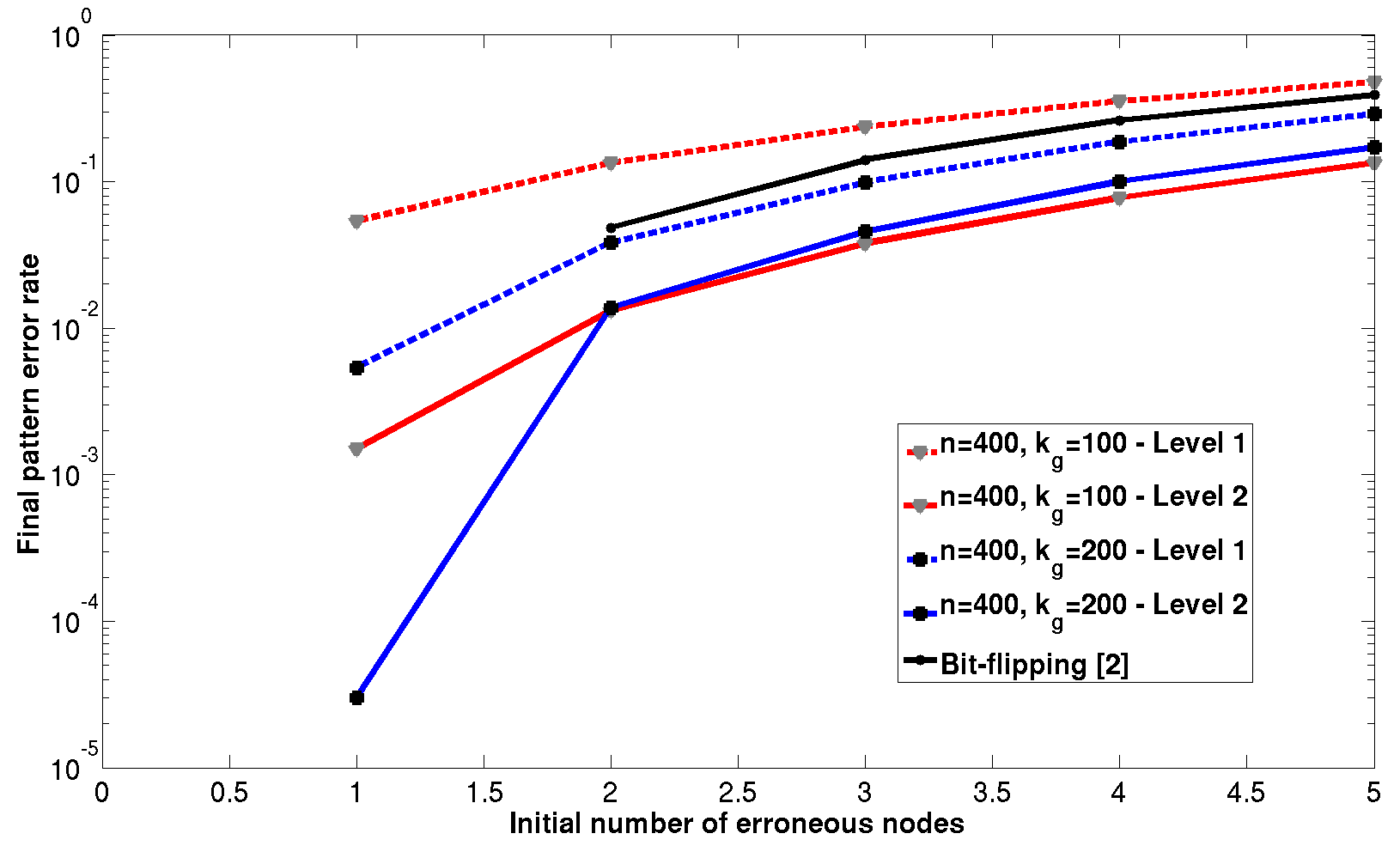}
\end{center}
\begin{center}
\vspace{-0.5cm}
\caption{Pattern error rate against the initial number of erroneous nodes \label{PER_full_redund}}
\end{center}
\vspace{-0.5cm}
\end{figure}

Table \ref{gain_error_rate} shows the gain we obtain by adding an additional second level to the network architecture. The gain is calculated as the ratio between the pattern error rate at the output of the first layer and the pattern error rate at the output of the second layer. 
\begin{table}[h]
\label{gain_error_rate}
\caption{Gain in Pattern Error Rate (PER) for different values of $n=400$ and initial number of errors}
\begin{center}
\begin{tabular}{ |p{2cm} | p{2cm}| p{2cm}| }
 \hline            
 Number of initial errors & Gain for $k_g=100$ & Gain for $k_g=200$ \\ \hline
 $2$ & $10.2$ & $2.79$ \\ 
 $3$ & $6.22$ & $2.17$ \\ 
 $4$ & $4.58$ & $1.88$ \\
 $5$ & $3.55$ & $1.68$ \\
 \hline 
\end{tabular}
\end{center}
\end{table}
%\begin{figure}[h]
%\begin{center}
%\includegraphics[width=.45\textwidth]{error_rate_gain.png}
%\end{center}
%\begin{center}
%\caption{The gain in bit and pattern error rates due to addition of a second layer \label{error_rate_gain}}
%\end{center}
%\end{figure}
\subsection{Comparison with Previous Work}
For the sake of completeness, table \ref{table_comparison} compares the proposed algorithm with the previous work from three different perspectives: the pattern retrieval capacity, the number of initial errors that can be corrected in the recall phase (denoted by $e$), the existence of an online iterative learning algorithm, and if there are any other restrictions such as the focus of the algorithm on particular patterns with some redundancy. In all cases it is assumed that the length of patterns is $n$.
\ctable[
caption = Neural associative memories compared together for a pattern of size $n$,
label = table_comparison,
pos = h,
]{|c|c|c|c|c|}{
\tnote[a]{The authors do not provide exact relationship for the pattern retrieval capacity. However, they show that for a particular setup with $n=2048$, we have $C = 60000$.}
\tnote[b]{PWR stands for Patterns With Redundancy.} 
\tnote[c]{$p$ is the order of correlations considered among patterns.}
\tnote[d]{$a>1$ is determined according to network parameters.}
\tnote[e]{EG stands for Expander Graphs.} 
}{
\hline 
 Algorithm & $C$ & $e$  & Learning? & Restrictions?\\ 
 \hline
 \cite{hopfield} & $O\left(n/\log(n)\right)$ & $O(n)$ & yes & no\\ 
 \hline
 \cite{venkatesh} & $n/2$ & $> 1$ & no & no\\ 
 \hline
 \cite{Jankowski} & $0.15n$ & $> 1$ & no & no\\ 
 \hline 
 \cite{Muezzinoglu1} & $n$ & $> 1$ & no & no\\ 
 \hline 
 \cite{gripon_sparse} & $>n$\tmark[a] & $> 1$ & yes & PWR\tmark[b]\\ 
 \hline 
 \cite{peretto} & $O\left(n^{p-2}\right)\tmark[c]$ & $> 1$ & yes & no\\ 
 \hline 
 \cite{KSS} & $O(a^n)$\tmark[b] & $> 2$ & no & PWR\tmark[d], EG\tmark[e]\\ 
 \hline 
 This paper & $O(a^n)$\tmark[b] & $> 1$ & yes & PWR\tmark[d]\\ 
 \hline  
 }

\section{Future Works}\label{section:conclusions}
In order to extend the multi-level neural network, we must first find a way to generate patterns that belong to a subspace with dimensions $nL-m_g$, where $m_g$ lies \emph{within} the inside of bounds $L(n-k)<m_g<nL-k$. This will give us a way to investigate the trade off between the maximum number of memorizable patterns and the degree of error correction possible. 

Furthermore, so far we have assumed that the second level enforces constraints in the same space. However, it is possible that the second level imposes a set of constraints in a totally different space. For this purpose, we need a mapping from one space into another. A good example is the written language. While they are local constraints on the spelling of the words, there are some constraints enforced by the grammar or the overall meaning of a sentence. The latter constraints are not on the space of letters but rather the space of grammar or meaning. Therefore, in order to for instance to correct an error in the word $\_at$, we can replace $\_$ with either $h$, to get \emph{hat}, or $c$ to get \emph{cat}. Without any other clue, we can not find the correct answer. However, let's say say we have the sentence "\emph{The $\_$at ran away}". Then from the constraints in the space of meaning we know that the subject must be an animal or a person. Therefore, we can return \emph{cat} as the correct answer. Finding a proper mapping is the subject of our future work.

\section*{Acknowledgment}

The authors would like to thank Prof. Amin Shokrollahi for helpful comments and discussions. This work was supported by Grant 228021-ECCSciEng of the European Research Council.

\begin{small}

\end{small}
\end{document}